\documentclass[11pt]{article}

\usepackage[T1]{fontenc}
\usepackage[margin=1in]{geometry}
\usepackage{amsmath,amssymb,amsfonts,mathtools,amsthm}
\usepackage{bm}
\usepackage{graphicx}
\usepackage{booktabs}
\usepackage{multirow}
\usepackage{enumitem}
\usepackage{url}
\usepackage{bbm}
\usepackage{placeins}
\usepackage{hyperref}
\usepackage{cleveref}

\newtheorem{theorem}{Theorem}[section]
\newtheorem{lemma}[theorem]{Lemma}

\theoremstyle{remark}

\newenvironment{keywords}{\par\smallskip\noindent\textbf{Keywords:}\ }{\par\smallskip}
\newenvironment{MSCcodes}{\par\noindent\textbf{MSC codes:}\ }{\par\smallskip}
\newcommand{\email}[1]{\href{mailto:#1}{\texttt{#1}}}

\crefname{theorem}{theorem}{theorems}
\crefname{lemma}{lemma}{lemmas}
\crefname{corollary}{corollary}{corollaries}
\crefname{figure}{figure}{figures}
\crefname{section}{section}{sections}
\crefname{appendix}{appendix}{appendices}
\crefname{equation}{equation}{equations}
\Crefname{theorem}{Theorem}{Theorems}
\Crefname{lemma}{Lemma}{Lemmas}
\Crefname{corollary}{Corollary}{Corollaries}
\Crefname{figure}{Figure}{Figures}
\Crefname{section}{Section}{Sections}
\Crefname{appendix}{Appendix}{Appendices}
\Crefname{equation}{Equation}{Equations}

\hypersetup{
  pdftitle={Approximating Uniform Random Rotations by Two-Stage Structured Hadamard Rotations in High Dimensions},
  pdfauthor={Tomer Zilca and Gal Mendelson}
}

\newcommand{\R}{\mathbb{R}}
\newcommand{\sphere}{S^{d-1}}
\newcommand{\E}{\mathbb{E}}
\newcommand{\Prob}{\mathbb{P}}

\newcommand{\Cov}{\mathrm{Cov}}
\newcommand{\Lip}{\mathrm{Lip}}
\newcommand{\Kol}{\mathrm{K}}
\newcommand{\Wone}{W_1}
\newcommand{\diag}{\mathrm{diag}}
\newcommand{\Normal}{\mathcal{N}}
\newcommand{\Unif}{\mathrm{Unif}}
\newcommand{\Had}{H}
\newcommand{\Tstr}{T}
\newcommand{\ip}[2]{\left\langle #1,#2\right\rangle}
\newcommand{\norm}[1]{\left\|#1\right\|_2}
\newcommand{\abs}[1]{\left|#1\right|}
\newcommand{\set}[1]{\left\{#1\right\}}

\newcommand{\DOne}{D^{(1)}}
\newcommand{\DTwo}{D^{(2)}}
\newcommand{\bbone}{\mathbbm{1}}

\title{Approximating Uniform Random Rotations by Two-Stage Structured Hadamard Rotations in High Dimensions}
\author{Tomer Zilca\thanks{Faculty of Data and Decision Sciences, Technion -- Israel Institute of Technology, Haifa, Israel (\email{tomerzilca@campus.technion.ac.il}).}
\and
Gal Mendelson\thanks{Edward P. Fitts Department of Industrial and Systems Engineering, North Carolina State University, Raleigh, NC, USA (\email{gmendel@ncsu.edu}).}
}

\begin{document}

\maketitle

\begin{abstract}
Uniform random rotations are a useful primitive in applications such as fast Johnson--Lindenstrauss embeddings, kernel approximation, communication-efficient learning, and recent AI compression pipelines, but they are computationally expensive to generate and apply in high dimensions. A common practical replacement is repeated structured random rotations built from Walsh--Hadamard transforms and random sign diagonals.

Applying the structured random rotation twice has been shown empirically to be useful, but the supporting theory is still limited. In this paper we study the approximation quality achieved when using this \emph{two-stage} structured Hadamard rotation. Here $d$ is the dimension of the rotated vector. Our results are both positive and negative. On the positive side, we prove explicit, nonasymptotic upper bounds of order $\sqrt{\log d/d}$ in both Kolmogorov and Wasserstein-1 distance for corresponding coordinates scaled by $\sqrt d$, uniformly over unit inputs and coordinates. For the input $u=e_1$, both distances also have lower bounds of order $d^{-1/2}$ at every coordinate. On the negative side, we prove an explicit lower bound on the Wasserstein-1 distance between the full vector distributions, showing a substantial discrepancy that does not vanish as $d$ grows for the input $u=e_1$. For the same input used in the lower bound, we also prove a matching asymptotic upper bound, showing that the lower-bound scale is sharp for that input.

Taken together, the results identify a clear separation between one-dimensional marginal behavior, where approximation improves with dimension, and the full-vector law, where a nonvanishing Wasserstein-1 discrepancy remains. The marginal bounds control coordinate threshold probabilities and expected scalar Lipschitz losses, including expected averages over coordinates. They do not establish joint-coordinate approximation, and the full-vector separation does not imply failure of a particular algorithm. Connecting these distributional results to performance guarantees for specific algorithms remains open.
\end{abstract}



\section{Introduction}

Uniform random rotations are a natural tool when one wishes to spread a deterministic vector evenly across directions. Let $R$ be Haar-distributed on the orthogonal group. If $u\in\R^d$, then $Ru$ is uniformly distributed on the sphere of radius $\norm{u}$. This symmetry is useful in a range of applications, including randomized numerical linear algebra, dimension reduction, derivative-free optimization, random feature generation, and communication-efficient learning \cite{AilonChazelle2009FJLT,AilonLiberty2009Fast,LeSarlosSmola2013Fastfood,YuEtAl2016ORF,ChoromanskiEtAl2019OMC,Vargaftik2021DRIVE,Vargaftik2022EDEN}.

The difficulty is computational. Generating and storing a dense random orthogonal matrix is expensive, and multiplying it by a vector costs $O(d^2)$ operations. In large-scale settings, this is often prohibitive. A practical alternative is to use \emph{structured random rotations}, built from cheap orthogonal transforms and a small amount of randomness. A particularly common construction uses Walsh--Hadamard matrices and diagonal random sign matrices. These transforms can be applied in $O(d\log d)$ time via the Fast Walsh--Hadamard Transform and require only $O(d)$ random bits \cite{FinoAlgazi1976FWHT,AilonChazelle2009FJLT,LeSarlosSmola2013Fastfood}.

A first question is whether a \emph{single-stage} Hadamard-sign transform can already produce an output law close to the uniform spherical law for every unit input. For an input vector $u$, this is given by $\frac{1}{\sqrt d}\Had D u$, where $d$ is the dimension, $\Had$ is a deterministic Hadamard matrix, and $D$ is a diagonal matrix with i.i.d.\ Rademacher random variables on the diagonal. As shown in \cite{Vargaftik2021DRIVE}, the answer is no. A discrete distribution can approximate the uniform law on the sphere only if it has many possible outputs, with probability mass sufficiently spread over the sphere. If $u=e_i$ is a basis vector, then $\frac{1}{\sqrt d}\Had D u$ is supported on only two antipodal points, namely the $\pm$ version of the $i$th column of $\Had$ scaled by $1/\sqrt d$. This obstruction motivates an additional randomized mixing stage, as used, for example, in DRIVE and EDEN \cite{Vargaftik2021DRIVE,Vargaftik2022EDEN}. For this basis input, the additional stage increases the support from two points to $2^d$ points. This raises the possibility of a better
approximation. We ask whether the two-stage output law approximates the uniform spherical law for all unit inputs, and in which sense.

The present paper studies the approximation quality of the \emph{two-stage} structured transform
\begin{equation}\label{eq:Tdef}
\Tstr(u)=\frac{1}{d}\Had D^{(1)}\Had D^{(2)}u,
\end{equation}
where $\Had$ is a Walsh--Hadamard matrix of order $d$ and $D^{(1)},D^{(2)}$ are independent diagonal matrices with i.i.d.\ $\pm 1$ entries. We discuss which properties of the Hadamard matrix and sign
diagonals the proofs require in \cref{app:transform-flexibility}. The motivating question is simple:

\medskip
\begin{quote}
How close is the distribution of $\Tstr(u)$ to the distribution of $Ru$?
\end{quote}
\medskip

This question appears naturally in several algorithmic developments. In DRIVE and EDEN, a rotation is applied before quantization in order to spread vectors more evenly across coordinates \cite{Vargaftik2021DRIVE,Vargaftik2022EDEN}. Those works analyze the resulting compression procedures; here we ask how closely the structured output law approximates the law of $Ru$ for a fixed input.

Closely related ideas now also appear in recent AI compression pipelines for high dimensional vectors, including KV-cache and vector quantization methods such as TurboQuant and RaBitQ \cite{Zandieh2025TurboQuant,gao2024rabitq}.
More broadly, Hadamard-based random transforms have long been used in fast Johnson--Lindenstrauss embeddings, Gaussian kernel approximation, and efficient randomized estimation \cite{AilonChazelle2009FJLT,AilonLiberty2009Fast,LeSarlosSmola2013Fastfood,YuEtAl2016ORF,ChoromanskiEtAl2019OMC,ChoromanskiRowlandWeller2017Kac}. In most such works, however, the transform is analyzed through the specific downstream algorithm rather than by directly comparing its law to the uniform spherical law. Our goal is to quantify both the sense in which its output law approximates the uniform spherical law and the sense in which it provably does not.

\subsection{What we prove}

We compare the two-stage structured rotation $\Tstr(u)$ with
$Ru$ at two levels: individual coordinates
and the full output vector.

For the coordinate comparison, we multiply each coordinate by
$\sqrt d$ so that it has variance one. Let $\Delta_{\mathrm K}(d)$
and $\Delta_{\mathrm W}(d)$ denote the largest Kolmogorov and
Wasserstein-1 distances between
$\sqrt d\,[\Tstr(u)]_k$ and $\sqrt d\,[Ru]_k$, over all unit
inputs $u$ and coordinates $k$. For every power-of-two $d\ge8$, we prove the following
nonasymptotic bounds:
\[
\begin{aligned}
\frac{e^{-1/24}}{\sqrt{2\pi d}}
&\le \Delta_{\mathrm K}(d)
\le 0.56\sqrt{\frac{3\log d}{d}}
     +\frac{2}{\sqrt d}+\frac{4}{d-4},\\[1ex]
\sqrt{\frac{5}{64\pi e d}}
&\le \Delta_{\mathrm W}(d)
\le \frac{3\sqrt{3\log d}+6+\sqrt 2}{\sqrt d}.
\end{aligned}
\] 
Thus the two distances vanish
for individual coordinates, with a remaining factor of order
$\sqrt{\log d}$ between the upper and lower bounds.

Second, for the full-vector distribution, we use Wasserstein-1
distance. Here the picture is fundamentally different. We prove an
explicit lower bound by analyzing the concrete input $u=e_1$. In
particular, our theorem gives
\[
\sup_{u\in\sphere}
\Wone\!\left(\Tstr(u),Ru\right)>
\begin{cases}
\displaystyle \frac13,
  & \text{for every power-of-two }d\ge256,\\[1ex]
0.6,
  & \text{for every power-of-two }d\ge32768.
\end{cases}
\]
More generally, the lower bound stays away from zero as $d$ grows
and approaches
$\sqrt{2(1-\sqrt{2/\pi})}$ $\approx0.6358$.
Thus the two-stage transform does not provide a vanishing full-vector Wasserstein-1 approximation to the law of $Ru$ uniformly over unit inputs.
Both outputs lie on the unit sphere, so their Wasserstein-1
distance is at most $2$.

We also sharpen this negative result for the same input $u=e_1$.
The lower bound uses a particular $1$-Lipschitz test function:
distance to the support of the structured output, which is a
rotated, scaled hypercube. A different test function could, in
principle, reveal a larger discrepancy for this input. We show
that it does not: this test function attains the full Wasserstein-1
distance for $u=e_1$. We also prove a matching asymptotic upper
bound, so the lower-bound scale is sharp for the input used in the
proof. We do not determine the largest full-vector distance over
all inputs.

Taken together, these results show that individual coordinate laws
approach their uniform-spherical counterparts, while the full output law remains
separated in Wasserstein-1 distance for $u=e_1$.

\subsection{Why this matters}

For one normalized output coordinate, the Kolmogorov bound compares
its probability of exceeding any prescribed threshold under the
two rotations. The Wasserstein-1 bound compares the expected value
of any scalar $1$-Lipschitz function. These are one-dimensional
guarantees. They do not compare the joint distribution of two
coordinates, even when each coordinate separately has a close
uniform-spherical counterpart.

A quantity can nevertheless involve many coordinates while its
expectation depends only on their individual laws. For example,
the expected fraction of normalized coordinates above a given
threshold is determined by the corresponding one-coordinate
probabilities. Our Kolmogorov bound controls the difference
between this expected fraction under the two rotations by at most
$\Delta_{\mathrm K}(d)$. Similarly, the Wasserstein-1 bound
controls the difference in expected averages of scalar
$1$-Lipschitz losses by at most $\Delta_{\mathrm W}(d)$.
Neither conclusion requires independence between coordinates.
We develop these examples in \cref{sec:implications}.

\subsection{Limitations}

The expected-average bounds do not compare the distributions of
those averages or establish their concentration. For $u=e_1$,
the full-vector lower bound shows that a $1$-Lipschitz function
of the whole output retains a substantial expectation difference;
it does not establish a performance difference for a particular
compression, embedding, or estimation algorithm. This paper
establishes distributional consequences and leaves their transfer
to specific algorithms open. Such a transfer requires identifying
the event or statistic that expresses an algorithm's guarantee and
determining whether the present bounds control it.

\section{Related work}\label{sec:related}

Structured orthogonal transforms built from Hadamard matrices and random sign diagonals have appeared in several distinct literatures. We separate prior work according to \emph{what} those transforms are used for and \emph{what kind of guarantee} is typically proved.

\paragraph{Fast approximate Gaussian and random feature generation}
This line of work uses Hadamard-based transforms to mimic properties of Gaussian randomness while preserving fast multiplication. Fastfood is a representative example in kernel approximation: it replaces dense Gaussian matrices by products of diagonal random matrices and Hadamard transforms in order to approximate random Fourier features much more efficiently \cite{LeSarlosSmola2013Fastfood}. Orthogonal Random Features follows a similar philosophy and studies how orthogonality can improve variance in kernel approximation \cite{YuEtAl2016ORF}. In these works, the main guarantee is not that the transform's output is close in law to $Ru$ for a fixed input $u$; rather, it is that the downstream estimator behaves favorably.

\paragraph{Fast dimension reduction and sketching}
This second line of work uses structured transforms for Johnson--Lindenstrauss embeddings and sketching \cite{AilonChazelle2009FJLT,AilonLiberty2009Fast}. The basic aim here is norm and inner-product preservation for finite point sets. More recent work studies structured orthogonal embeddings and related random matrix ensembles, often quantifying estimator variance, concentration, or embedding distortion \cite{ChoromanskiEtAl2018Structured,ChoromanskiRowlandWeller2017Kac}. These papers provide strong algorithmic guarantees, but the object of study is typically a task-specific statistic such as norm distortion, not the full-vector distributional distance to the law of $Ru$.

A central example in randomized numerical linear algebra is the subsampled randomized Hadamard transform (SRHT), which samples coordinates after one Hadamard-sign stage; Tropp proves subspace-embedding guarantees for this construction \cite{Tropp2011SRHT}. Writing $Q=d^{-1/2}\Had$, sampling $r$ coordinates after our two stages gives $\sqrt{d/r}\,P QD^{(1)}QD^{(2)}$, the \emph{rerandomized SRHT} (SRRHT), up to the transpose convention in \cite{amsel2026openquestions}. Yang proves a subspace-embedding guarantee for this sampled two-stage construction \cite{yang2026rerandomizedsrht}. We use \emph{two-stage} for the nonsubsampled rotation to specify the number of independently randomized Hadamard factors; other stage counts have also been studied \cite{ChoromanskiEtAl2019OMC,YuEtAl2016ORF}. Our marginal bounds apply to each selected coordinate, but do not compare the joint law of the selected coordinates.

\paragraph{Derivative-free optimization and Monte Carlo variance reduction}
Structured orthogonal matrices have also been used in randomized optimization and Monte Carlo settings. Choromanski et al.\ study orthogonal Monte Carlo methods and show that structured orthogonal constructions can improve estimator behavior while remaining computationally efficient \cite{ChoromanskiEtAl2019OMC}. Again, the central comparison is between estimator variances or optimization performance, not between the output law of the transform and the law of $Ru$ for a fixed input $u$.

\paragraph{Compression and quantization}
The most direct motivation for our work comes from communication efficient learning and quantization. DRIVE and EDEN both rely on a Hadamard-based rotation before quantization to spread information more evenly across coordinates \cite{Vargaftik2021DRIVE,Vargaftik2022EDEN}. More recently, orthogonal or Hadamard-type rotations have continued to appear in vector and AI compression pipelines, including TurboQuant and RaBitQ \cite{Zandieh2025TurboQuant,gao2024rabitq}. The rotate-then-quantize construction used in these later methods was developed in DRIVE and EDEN \cite{Vargaftik2021DRIVE,Vargaftik2022EDEN,BenBasatEtAl2026TurboQuantNote,vLLM2026TurboQuantImplementation}. RaBitQ's randomly rotated hypercube codebook amounts to taking coordinate signs after a rotation \cite{gao2024rabitq}, while TurboQuant$_{\mathrm{mse}}$ uses the EDEN quantizer with a fixed scale \cite{BenBasatEtAl2026TurboQuantNote}.

\section{Setup and preliminaries}\label{sec:setup}

Throughout, we assume $d = 2^m$ for some $m\in\mathbb N$ to ensure the existence of a Walsh--Hadamard matrix of order $d$. We denote by $\Had\in\{-1,1\}^{d\times d}$ the unnormalized Walsh--Hadamard matrix, defined recursively by $H_1=(1)$ and
\[
H_{2n}=\begin{pmatrix}H_n&H_n\\ H_n&-H_n\end{pmatrix}.
\]
It is symmetric and satisfies
$\Had \Had^\top = d I_d.$
Consequently, the scaled matrix $d^{-1/2}\Had$ is orthogonal.
Let
\[
\sphere \coloneqq \set{x\in\R^d:\norm{x}=1}
\]
denote the unit sphere. Let $\DOne$ and $\DTwo$ be independent diagonal matrices whose diagonal entries are i.i.d.\ Rademacher random variables. 

For a deterministic input $u\in\sphere$, we define the structured random rotation
\[
\Tstr(u)\coloneqq \frac{1}{d}\Had \DOne \Had \DTwo u.
\]
We compare the law of $\Tstr(u)$ to that of
\[
X(u)\coloneqq Ru,
\]
where $R$ is Haar distributed on the orthogonal group $\mathcal O(d)$. By rotational invariance of Haar measure, $Ru$ is uniformly distributed on $\sphere$ whenever $u\in\sphere$ is deterministic.

\subsection{Probability metrics}\label{subsec:probability-metrics}

We compare probability distributions at two levels: the law of a
single coordinate and the joint law of the full vector. We use
Kolmogorov distance for the one-coordinate comparison and
Wasserstein-1 distance for both comparisons. 

\paragraph{Kolmogorov distance}
For random vectors $A,B\in\R^q$, define the multivariate
Kolmogorov distance by
\[
\Kol_q(A,B)
=
\sup_{t\in\R^q}
\abs{
\Prob(A_1\le t_1,\ldots,A_q\le t_q)
-
\Prob(B_1\le t_1,\ldots,B_q\le t_q)
};
\]
see, for example, \cite[Section~1]{GauntLi2023Metrics}.
This distance measures the largest discrepancy in probabilities
of coordinatewise lower orthants. Each event simultaneously
constrains all coordinates, so the comparison concerns their
joint distribution, including their dependence, in the chosen
coordinate system.

We do not analyze this multivariate Kolmogorov distance in the
present paper. Our Kolmogorov results concern single-coordinate
marginals; determining the approximation behavior of the joint
law in multivariate Kolmogorov distance is left open.

For real-valued random variables $A,B$, the definition reduces to
the usual Kolmogorov distance, which we denote without a subscript:
\[
\Kol(A,B)
\coloneqq \Kol_1(A,B)
=
\sup_{t\in\R}
\abs{\Prob(A\le t)-\Prob(B\le t)}.
\]
Thus $\Kol(A,B)$ is the largest difference between their cumulative
distribution functions. Closeness in this metric means that the two distributions
assign nearly the same probability to falling below any given
threshold, providing uniform control of differences in both
lower- and upper-tail probabilities. In \cref{sec:positive}, we apply this
distance to corresponding coordinates of $\Tstr(u)$ and $X(u)=Ru$.
The resulting bounds describe agreement of individual marginal
distributions, uniformly over unit inputs and coordinates; they
do not establish agreement of the joint distribution of two or
more coordinates.

\paragraph{Wasserstein-1 distance}
For random vectors $A,B\in\R^q$ with finite first moments, define
\[
\Wone(A,B)
=
\inf_{\substack{
A'\stackrel{\mathrm{law}}{=}A\\
B'\stackrel{\mathrm{law}}{=}B
}}
\E\!\left[\norm{A'-B'}\right],
\]
where the infimum is over all joint distributions of $(A',B')$
with the prescribed marginal laws. Such a joint distribution
is called a coupling. The distance is therefore the smallest
expected Euclidean separation achievable by coupling the two
distributions.

Equivalently, Kantorovich--Rubinstein duality
\cite{Villani2009} gives
\[
\Wone(A,B)
=
\sup_{f\in\Lip(1)}
\abs{\E[f(A)]-\E[f(B)]},
\]
where $\Lip(1)$ denotes the class of functions $f:\R^q\to\R$
satisfying
\[
\abs{f(x)-f(y)}\le\norm{x-y}
\qquad\text{for all }x,y\in\R^q.
\]
This formulation expresses the same distance as the largest
discrepancy in expectations over $1$-Lipschitz functions.
Here the functions in $\Lip(1)$ act on $\R$ when comparing
single coordinates, and on $\R^d$ when comparing full vectors.
Closeness in this metric means that samples from the two
distributions can be paired so that, on average, paired samples
are close: their absolute difference is small for individual
coordinates, or their Euclidean distance is small for full vectors.
It also means that any $1$-Lipschitz function---a function whose
output changes by at most the distance between its inputs---has
nearly the same expectation under the two distributions.

When $q=1$, the Euclidean cost becomes $|x-y|$, and hence
\[
\Wone(A,B)
=
\inf_{\substack{
A'\stackrel{\mathrm{law}}{=}A\\
B'\stackrel{\mathrm{law}}{=}B
}}
\E\abs{A'-B'}.
\]
In this case, Vallender's formula \cite{Vallender1974Wasserstein}
also gives
\[
\Wone(A,B)
=
\int_{\R}
\abs{\Prob(A\le t)-\Prob(B\le t)}\,dt.
\]
Thus, in one dimension, Kolmogorov distance measures the largest
CDF discrepancy, while Wasserstein-1 distance measures its
integral over the real line.

We use the same notation $\Wone$ for scalar and vector arguments.
In \cref{sec:positive}, it compares single-coordinate laws with
absolute-distance cost. In \cref{sec:negative}, it compares the
full-vector laws of $\Tstr(u)$ and $X(u)$ with Euclidean cost.
The latter comparison concerns the joint distribution and its
geometry. A coordinate projection is $1$-Lipschitz, so the dual
representation immediately gives
\[
\Wone(A_k,B_k)\le\Wone(A,B).
\]
However, closeness of the individual coordinate laws does not
by itself imply closeness of the joint laws. Consequently,
vanishing marginal distances are compatible with the
nonvanishing worst-case full-vector Wasserstein discrepancy
established in \cref{sec:negative}.

\paragraph{Coordinate normalization}
For the one-dimensional Wasserstein comparisons in our positive
results, we multiply each coordinate by $\sqrt d$.
Indeed, for $u\in\sphere$, each coordinate of $\Tstr(u)$ and $X(u)$
has mean zero and variance $1/d$, as shown below, so this normalization
gives variance one.
The normalization does not affect the Kolmogorov comparisons, since
Kolmogorov distance is invariant under a common positive rescaling.
No such scaling is needed for the multidimensional negative results,
since both full vectors already have Euclidean norm one.

\paragraph{Connections between the one-dimensional metrics}

Kolmogorov and Wasserstein-1 distances in one dimension are related, but general
comparisons between them can lose information about the convergence
rate. If a real-valued random variable $V$ has a density bounded by
$M$, then
\[
\Kol(U,V)\le \sqrt{2M\,\Wone(U,V)};
\]
see~\cite[Equation~(1.2)]{GauntLi2023Metrics}.
Conversely, if $\E[U^2]=\E[V^2]=1$, the one-dimensional CDF
representation of Wasserstein-1 distance~\cite{Vallender1974Wasserstein}
gives
\[
\Wone(U,V)\le 2\sqrt{\Kol(U,V)}.
\]
For the latter inequality, one integrates the CDF difference over
$[-R,R]$, bounds the remaining tails using the second moments,
and optimizes over $R$.

These comparisons apply to the normalized coordinates considered
in \cref{sec:positive}: both have second moment one, and the
normalized spherical coordinate has a density bounded uniformly
in $d\ge8$. Each conversion takes a square root of the error.
For example, our Kolmogorov upper bound of order
$O(\sqrt{\log d/d})$ would yield only a Wasserstein-1 upper bound
of order $O((\log d/d)^{1/4})$ through the second inequality.
Our direct Wasserstein-1 analysis instead proves the stronger
bound $O(\sqrt{\log d/d})$.
Similarly, converting our Wasserstein-1 upper bound through the
first inequality yields a weaker rate than our direct Kolmogorov
upper bound.
The separate lower bounds also improve on these general conversions:
we prove a lower bound of order $d^{-1/2}$ in each metric,
whereas transferring either lower bound to the other metric
through the displayed inequalities yields only order $d^{-1}$.
Thus all four bounds in \cref{thm:positive} are stronger than
those obtained by transferring the corresponding bound from
the other metric using these comparisons.

\subsection{A basic fact on the uniform spherical law}

Let $\mu_d$ denote the uniform probability measure on $\sphere$. If $G\sim \Normal(0,I_d)$, then
\[
\frac{G}{\norm{G}}\sim \mu_d.
\]
We use this Gaussian representation repeatedly.

\subsection{Matching first and second moments}

For every deterministic $u\in\sphere$, the structured transform
$\Tstr(u)$ has the same first and second moments as the reference vector $Ru$:
\[
\E[\Tstr(u)]=0,
\qquad
\Cov(\Tstr(u))=\frac{1}{d}I_d.
\]
Indeed, independence and $\E[\DOne]=0$ give
\[
\E[\Tstr(u)]
=\frac{1}{d}\Had\,\E[\DOne]\Had\,\E[\DTwo]u=0.
\]
For the covariance, write $v=\Had\DTwo u$. Independence of the
Rademacher signs gives, for every $j$,
\[
\E[v_j^2]=\sum_{\ell=1}^d H_{j\ell}^2u_\ell^2
=\norm{u}^2=1,
\]
while conditioning on $\DTwo$ yields
\[
\E[\DOne vv^\top\DOne\mid\DTwo]
=\diag(v_1^2,\ldots,v_d^2).
\]
Consequently,
\[
\Cov(\Tstr(u))
=\frac{1}{d^2}\Had\,
  \E[\diag(v_1^2,\ldots,v_d^2)]\Had^\top
=\frac{1}{d^2}\Had\Had^\top
=\frac{1}{d}I_d.
\]

These identities provide a basic reason to regard the structured
output as a plausible surrogate for $Ru$:
both output distributions are centered and have covariance matrix
$d^{-1}I_d$. Matching these moments, however, does not establish
closeness of the distributions themselves. The question we study
is in what sense, and to what extent, this agreement extends to
their distributional behavior.

\section{Positive results: marginal approximation in two metrics}\label{sec:positive}

We compare one coordinate of the structured rotation with the corresponding
coordinate of $Ru$. The bounds hold uniformly over unit inputs and
coordinates; the lower bounds are using a single input.
Here $\Wone$ is applied to real-valued coordinates with cost $|x-y|$;
\cref{sec:negative} applies it to full vectors with Euclidean cost.
For later use, write
\[
S_{u,k}:=\sqrt d\,[\Tstr(u)]_k,
\qquad
Y_{u,k}:=\sqrt d\,[X(u)]_k.
\]
These scaled variables both have second moment one.
Define the worst-case coordinate distances
\[
\begin{aligned}
\Delta_{\mathrm K}(d)&:=
\sup_{u\in\sphere}\sup_{1\le k\le d}
\Kol(S_{u,k},Y_{u,k}),\\
\Delta_{\mathrm W}(d)&:=
\sup_{u\in\sphere}\sup_{1\le k\le d}
\Wone(S_{u,k},Y_{u,k}).
\end{aligned}
\]

\begin{theorem}[One-coordinate approximation]\label{thm:positive}
Let $d=2^m\ge 8$. Then the following bounds hold.
\begin{enumerate}[label=\textup{(\roman*)}]
\item \emph{Kolmogorov distance:}
\begin{equation}\label{eq:marginal-k-bounds}
\frac{e^{-1/24}}{\sqrt{2\pi d}}
\le \frac{1}{2^{d+1}}\binom{d}{d/2}
\le \Delta_{\mathrm K}(d) \le
C_{\mathrm{BE}}\sqrt{\frac{3\log d}{d}}
+\frac{2}{\sqrt d}+\frac{4}{d-4},
\end{equation}
where $C_{\mathrm{BE}}$ is any absolute constant valid in the
Berry--Esseen inequality for independent summands; for example,
one may take $C_{\mathrm{BE}}=0.56$~\cite{Shevtsova2013BerryEsseen}.

\item \emph{Wasserstein-1 distance:}
\begin{equation}\label{eq:marginal-w-bounds}
\frac{1}{\sqrt d}\sqrt{\frac{5}{64\pi e}}
\le \Delta_{\mathrm W}(d) \le
\frac{3\sqrt{3\log d}+6+\sqrt 2}{\sqrt d}.
\end{equation}
\end{enumerate}
In both parts the lower bound holds with $u=e_1$ for every $k$.
\end{theorem}

\subsection{Interpretation}

The bounds in \cref{thm:positive} compare normalized coordinates,
each with mean zero and variance one. In both metrics, the
worst-case distance is bounded below by a constant times
$d^{-1/2}$ and above by a constant times $\sqrt{\log d/d}$.
Thus the upper and lower bounds identify the same polynomial
scale, with a remaining gap of order $\sqrt{\log d}$.
For Kolmogorov distance, Robbins's explicit Stirling
bounds~\cite{Robbins1955Stirling} give
\[
2^{-d-1}\binom{d}{d/2}\sim(2\pi d)^{-1/2}.
\]
Both lower bounds use the input $u=e_1$ for every
coordinate, whereas the upper bounds hold uniformly over all
unit inputs and coordinates.

The logarithmic factor in the upper bounds arises from
controlling the largest absolute coordinate after the first
randomized Hadamard stage. Our results do not establish whether
this factor is necessary.
We prove upper and lower bounds separately in each metric;
all four bounds are stronger than those obtained by transferring
the corresponding bound from the other metric through the
general comparisons discussed in \cref{subsec:probability-metrics}.

\subsection{Possible implications}
For Kolmogorov distance, consider a rule that examines whether one
normalized coordinate exceeds a fixed threshold $a$: it compares
$\sqrt d[\Tstr(u)]_k$ with $a$, rather than examining other coordinates.
By \eqref{eq:marginal-k-bounds}, replacing the structured rotation by a Haar
rotation changes this threshold probability by at most the stated upper
bound. This is a statement about the rotation's output distribution; it
does not by itself bound the error of a downstream algorithm. In
particular, the same estimate does not control a decision based jointly
on two coordinates, such as both coordinates exceeding specified
thresholds.

For Wasserstein-1 distance, consider the expected value of a
$1$-Lipschitz scalar loss $f(\sqrt d[\Tstr(u)]_k)$, for example a clipped
absolute deviation from a fixed target. The difference from the corresponding expectation under the uniform spherical law is at most the upper bound
in \eqref{eq:marginal-w-bounds}. The Lipschitz condition is essential to
this direct use of the Wasserstein bound: a discontinuous threshold
indicator does not qualify. This one-coordinate bound does not directly control a general Lipschitz function of two or more coordinates, although it does control the expectation of a separable average of scalar Lipschitz functions, as shown in \cref{sec:implications}.
Neither example asserts that a particular compression, embedding, or
estimation algorithm satisfies these conditions without further analysis.

\subsection{Proof roadmap}

\paragraph{Upper bounds}

We first condition on $\DTwo$. This fixes the
intermediate vector and expresses each normalized output coordinate
as a weighted sum of independent random signs, with mean zero and
variance one. Such a sum is close to a standard normal when no
individual coefficient is large enough to dominate it.

Berry--Esseen quantifies this approximation in Kolmogorov distance,
while a Stein bound for independent summands does so in
Wasserstein-1 distance. Both bounds depend on the absolute third
moments of the summands, which we control using the largest absolute
coefficient. For any fixed unit input, $d^{-1/2}\Had\DTwo u$ has small
coordinates with high probability, making this coefficient small. We apply the normal
approximation bounds on these typical sign choices and bound the
contribution of the remaining choices separately.

Finally, the normalized coordinate $Y_{u,k}$ is also close
to a standard normal. Using the normal distribution as a common
reference and applying the triangle inequality gives the desired
comparison between the structured and uniform-spherical coordinates in each metric.

\paragraph{Lower bounds}
We choose the first standard basis vector as the input. Every
normalized coordinate of the structured transform then has a
scaled binomial law, whose discrete structure creates an unavoidable
discrepancy from the continuous uniform-spherical coordinate distribution.

For Kolmogorov distance, we use the probability mass at zero.
This produces a jump in the structured coordinate's cumulative
distribution function, whereas the uniform-spherical coordinate's cumulative
distribution function is continuous. Their maximum difference
must therefore be at least half the size of this jump.

For Wasserstein-1 distance, we use the regularly spaced lattice
support of the structured coordinate. Distance to this lattice
is a $1$-Lipschitz function: its expectation is zero under the
structured law and positive under the uniform-spherical law, which assigns
probability to the gaps between lattice points. Bounding this
expectation from below gives the Wasserstein-1 lower bound.

\subsection{Proofs of the marginal bounds}
\paragraph{Common notation}
Fix $u\in\sphere$ and $k\in\{1,\ldots,d\}$, and let
$G\sim\Normal(0,1)$. Recall $S_{u,k}$ and $Y_{u,k}$ from above.
For $\lambda\in\{\pm1\}^d$, let
\[
E_\lambda:=\{\DTwo=\diag(\lambda)\},
\qquad
v^\lambda:=d^{-1/2}\Had\diag(\lambda)u.
\]
The events $E_\lambda$ are disjoint, have probability $2^{-d}$, and
partition the sample space. With $\varepsilon_i:=D^{(1)}_{ii}$, on each
$E_\lambda$ we have
\begin{equation}\label{eq:marginal-conditional-sum}
S_{u,k}=\sum_{i=1}^d H_{ki}\varepsilon_i v_i^\lambda,
\qquad
\sum_{i=1}^d(v_i^\lambda)^2=1.
\end{equation}
The $\varepsilon_i$ remain independent Rademacher variables on every
$E_\lambda$, since $\DOne$ and $\DTwo$ are independent.

We use the same flattening estimate for both upper bounds. For each $i$,
$[d^{-1/2}\Had\DTwo u]_i$ is a Rademacher sum with squared
coefficient sum $1/d$.
Its moment generating function is at most $\exp(t^2/(2d))$, by
$\cosh x\le\exp(x^2/2)$. Applying the exponential Markov inequality to
both tails and taking a union bound gives, for $s>0$,
\begin{equation}\label{eq:marginal-flattening}
\Prob\!\left(\left\|d^{-1/2}\Had\DTwo u\right\|_\infty\ge s\right)
\le 2d\exp\!\left(-\frac{ds^2}{2}\right).
\end{equation}
This is the standard randomized Hadamard flattening step; see also
\cite{AilonChazelle2009FJLT}. Write
$\mathcal G_s=\{\lambda:\|v^\lambda\|_\infty<s\}$ and
$\mathcal B_s=\{\pm1\}^d\setminus\mathcal G_s$.

\paragraph{Kolmogorov upper bound}
On $E_\lambda$, the summands in \eqref{eq:marginal-conditional-sum}
are independent and centered, with total variance one and total third
absolute moment $\sum_i|v_i^\lambda|^3$. The Berry--Esseen inequality
\cite[Section~5]{ChenShao2005Stein} therefore gives
\[
\delta_\lambda:=\sup_{t\in\R}
\abs{\Prob(S_{u,k}\le t\mid E_\lambda)-\Prob(G\le t)}
\le C_{\mathrm{BE}}\sum_{i=1}^d|v_i^\lambda|^3
\le C_{\mathrm{BE}}\|v^\lambda\|_\infty.
\]
For each fixed $t$, conditioning on the disjoint events $E_\lambda$
and separating the good and bad sign choices yields
\begin{align*}
\abs{\Prob(S_{u,k}\le t)-\Prob(G\le t)}
&\le\sum_{\lambda\in\mathcal G_s}2^{-d}\delta_\lambda
     +\sum_{\lambda\in\mathcal B_s}2^{-d}\delta_\lambda\\
&\le C_{\mathrm{BE}}s
    +\Prob\!\left(\left\|d^{-1/2}\Had\DTwo u\right\|_\infty\ge s\right)\\
&\le C_{\mathrm{BE}}s+2d e^{-ds^2/2},
\end{align*}
where $\delta_\lambda\le1$ was used on $\mathcal B_s$. Taking the
supremum over $t$ gives the same bound for $\Kol(S_{u,k},G)$.

The one-coordinate spherical-to-normal bound of Diaconis and Freedman
\cite[inequality~(1)]{DiaconisFreedman1987} gives, for $d\ge8$,
\[
\Kol(Y_{u,k},G)\le\frac4{d-4}.
\]
Here their variation norm is twice the supremum difference over
measurable sets, which accounts for the constant $4$ in our convention.
The triangle inequality, followed by $s=\sqrt{3\log d/d}$, gives
\[
\Kol(S_{u,k},Y_{u,k})
\le C_{\mathrm{BE}}\sqrt{\frac{3\log d}{d}}
    +\frac2{\sqrt d}+\frac4{d-4}.
\]
Kolmogorov distance is invariant under a common positive scaling, so
this is the upper bound in \eqref{eq:marginal-k-bounds}.

\paragraph{Wasserstein upper bound}
We reuse \eqref{eq:marginal-conditional-sum} and
\eqref{eq:marginal-flattening}. For sums of independent variables, the
normal approximation bound from Stein's method
\cite[Theorem~3.2]{ChenShao2005Stein} gives, for any fixed
$1$-Lipschitz $f:\R\to\R$ and each $E_\lambda$,
\[
\abs{\E[f(S_{u,k})\mid E_\lambda]-\E f(G)}
\le 3\sum_{i=1}^d|v_i^\lambda|^3
\le 3\|v^\lambda\|_\infty.
\]
Condition on the same
disjoint events, using $\|v^\lambda\|_\infty\le1$ on
$\mathcal B_s$. The preceding bound and \eqref{eq:marginal-flattening}
give
\begin{align*}
\abs{\E f(S_{u,k})-\E f(G)}
&\le\sum_{\lambda\in\{\pm1\}^d}2^{-d}
   \abs{\E[f(S_{u,k})\mid E_\lambda]-\E f(G)}\\
&\le 3s\sum_{\lambda\in\mathcal G_s}\Prob(E_\lambda)
    +3\sum_{\lambda\in\mathcal B_s}\Prob(E_\lambda)\\
&\le 3s+3\sum_{\lambda\in\mathcal B_s}\Prob(E_\lambda)
\le 3s+6d e^{-ds^2/2}.
\end{align*}
Taking the supremum over $f$ yields the same bound for
$\Wone(S_{u,k},G)$.

For the comparison with the uniform spherical law, let $Z=(Z_1,\ldots,Z_d)$ have independent
standard normal coordinates and put $Q=\|Z\|_2$. Then
$G'=Z_k$ has the law of $G$, while $Y'=\sqrt d\,Z_k/Q$ has the law
of $Y_{u,k}$. Using this coupling in the primal representation of
Wasserstein-1 distance \cite{Villani2009}, followed by
Cauchy--Schwarz, gives
\begin{align*}
\Wone(Y_{u,k},G)
&\le \E\abs{Y'-G'}
 =\E\!\left[\frac{|Z_k|}{Q}\abs{\sqrt d-Q}\right]\\
&\le\left(\E\frac{Z_k^2}{Q^2}\right)^{1/2}
      \left(\E(Q-\sqrt d)^2\right)^{1/2}.
\end{align*}
Symmetry gives $\E[Z_k^2/Q^2]=1/d$. Also,
\[
\E(Q-\sqrt d)^2
=\E\frac{(Q^2-d)^2}{(Q+\sqrt d)^2}
\le\frac1d\E(Q^2-d)^2=2,
\]
because $Q^2=\sum_i Z_i^2$ has mean $d$ and variance $2d$.
Thus $\Wone(Y_{u,k},G)\le\sqrt{2/d}$.
The triangle inequality and the same choice
$s=\sqrt{3\log d/d}$ imply
\[
\Wone(S_{u,k},Y_{u,k})
\le 3\sqrt{\frac{3\log d}{d}}
      +\frac6{\sqrt d}+\sqrt{\frac2d}.
\]
This is the upper bound in \eqref{eq:marginal-w-bounds}, which is
stated for the scaled coordinates.
The dependence of the marginal upper bounds on the threshold $s$,
including their minimizing thresholds, is illustrated in
\cref{app:marginal-upper-thresholds}.

\paragraph{Kolmogorov lower bound}
Take $u=e_1$. On every $E_\lambda$,
$v_i^\lambda=\lambda_1H_{i1}/\sqrt d$. Hence
\eqref{eq:marginal-conditional-sum} has, conditionally on each
$E_\lambda$, the law of $d^{-1/2}\sum_{i=1}^d\varepsilon_i$.
The events $E_\lambda$ partition the sample space, so the
unconditional law is the same. As $d$ is even, it is symmetric and has
an atom at zero of mass
\[
p_d=\Prob(S_{e_1,k}=0)=2^{-d}\binom{d}{d/2}.
\]
Thus $\Prob(S_{e_1,k}\le0)=(1+p_d)/2$. The spherical coordinate
$Y_{e_1,k}$ has a continuous symmetric law, so
$\Prob(Y_{e_1,k}\le0)=1/2$. Evaluating the two distribution functions
at zero gives
\[
\Kol(S_{e_1,k},Y_{e_1,k})\ge\frac{p_d}{2}
=\frac{1}{2^{d+1}}\binom{d}{d/2}.
\]
Scaling invariance proves the lower bound in
\eqref{eq:marginal-k-bounds} involving the central binomial coefficient.
To make the order and constant explicit, write $d=2n$. Robbins's
two-sided Stirling estimate \cite{Robbins1955Stirling} states that
\[
j!=\sqrt{2\pi j}\left(\frac{j}{e}\right)^j e^{r_j},
\qquad
\frac{1}{12j+1}<r_j<\frac{1}{12j}
\quad(j\ge1).
\]
Applying the lower estimate to $(2n)!$ and the upper estimate to
both factors of $n!$ gives
\begin{align*}
\frac{p_d}{2}
&=\frac{1}{2\sqrt{\pi n}}\exp(r_{2n}-2r_n)\\
&\ge\frac{1}{2\sqrt{\pi n}}
  \exp\!\left(\frac{1}{24n+1}-\frac{1}{6n}\right)\\
&=\frac{1}{\sqrt{2\pi d}}
  \exp\!\left(\frac{1}{12d+1}-\frac{1}{3d}\right).
\end{align*}
This proves the remaining inequality in \eqref{eq:marginal-k-bounds}.
In particular, because $d\ge8$, it is at least
$e^{-1/24}/\sqrt{2\pi d}$.

\paragraph{Wasserstein lower bound}
Keep $u=e_1$. The same binomial representation shows that
$S_{e_1,k}$ lies almost surely on $h\mathbb Z$, where
$h=2/\sqrt d$. Let
$g_h(y)=\operatorname{dist}(y,h\mathbb Z)$.
For every coupling $(S',Y')$ of $S_{e_1,k}$ and $Y_{e_1,k}$,
$|S'-Y'|\ge g_h(Y')$ almost surely. The coupling representation
of Wasserstein-1 distance therefore gives
\begin{equation}\label{eq:marginal-lattice-witness}
\Wone(S_{e_1,k},Y_{e_1,k})\ge\E g_h(Y_{e_1,k}).
\end{equation}

The density of the scaled spherical coordinate $Y_{e_1,k}$ is
\cite[Section~2]{DiaconisFreedman1987}
\[
f_d(y)=A_d\left(1-\frac{y^2}{d}\right)^{(d-3)/2}
\mathbf 1_{\{|y|<\sqrt d\}},\qquad
A_d^{-1}=\int_{-\sqrt d}^{\sqrt d}
\left(1-\frac{y^2}{d}\right)^{(d-3)/2}dy.
\]
We need only a simple density floor on $[-1,1]$. Using
$\log(1-z)\le-z$ and extending a Gaussian integral to $\R$,
\[
A_d^{-1}
\le\int_{\R}\exp\!\left(-\frac{d-3}{2d}y^2\right)dy
=\sqrt{\frac{2\pi d}{d-3}},
\qquad
A_d\ge\sqrt{\frac{d-3}{2\pi d}}
\ge\sqrt{\frac5{16\pi}}.
\]
For $|y|\le1$, the elementary inequality
$\log(1-z)\ge-z/(1-z)$, with $z=1/d$, gives
\[
\left(1-\frac{y^2}{d}\right)^{(d-3)/2}
\ge\left(1-\frac1d\right)^{(d-3)/2}
\ge e^{-1/2}.
\]
Normalization also gives
$1\ge 2A_d e^{-1/2}$, so $M_d=A_d\le\sqrt e/2$, as used in
the interpretation.
Consequently,
$f_d(y)\ge \sqrt{5/(16\pi e)}$ on $[-1,1]$.

On each full grid cell $[jh,(j+1)h]$,
$\int_{jh}^{(j+1)h}g_h(y)\,dy=h^2/4$.
Since $h\le1$, each of $[0,1]$ and $[-1,0]$ contains
$N=\lfloor1/h\rfloor\ge1/(2h)$ full cells. Therefore
\[
\int_{-1}^1g_h(y)\,dy
\ge 2N\frac{h^2}{4}\ge\frac h4.
\]
Combining this with \eqref{eq:marginal-lattice-witness} yields
\[
\Wone(S_{e_1,k},Y_{e_1,k})
\ge\int_{-1}^1g_h(y)f_d(y)\,dy
\ge\sqrt{\frac5{16\pi e}}\frac h4
=\sqrt{\frac5{64\pi e}}\,d^{-1/2}.
\]
This proves the lower bound in \eqref{eq:marginal-w-bounds} and
completes the proof of
\cref{thm:positive}.

\section{Negative result: global discrepancy}\label{sec:negative}

We now turn to the full vector distribution. In contrast with the positive result for one-dimensional marginals, the full structured law remains globally separated from the uniform spherical law in Wasserstein distance.

\begin{theorem}[Explicit Wasserstein lower bound]\label{thm:negative-explicit}
Let
\[
m_d \coloneqq \sqrt{\frac{2}{\pi}}\sqrt{\frac{d}{d-1}}.
\]
For every power-of-two dimension $d\ge 8$ and every $t\in [0,1-m_d],$
we have
\[
\sup_{u\in \sphere}\Wone\!\left(\Tstr(u),X(u)\right)
\ge
\sqrt{2(1-m_d-t)}\,(1-e^{-(d-1)t^2/2}).
\]
\end{theorem}

\begin{theorem}[Nonvanishing Wasserstein discrepancy]\label{cor:negative-concrete}
The following lower bounds hold:
\[
\sup_{u\in \sphere}\Wone\!\left(\Tstr(u),X(u)\right)
>
\begin{cases}
\displaystyle \frac13, & \text{for every power-of-two } d\ge 256,\\[1.2ex]
0.6, & \text{for every power-of-two } d\ge 32768.
\end{cases}
\]
Moreover,
\[
\liminf_{d\to\infty,\ d\in\{2^m:m\in\mathbb N\}}
\sup_{u\in \sphere}\Wone\!\left(\Tstr(u),X(u)\right)
\ge
\sqrt{2\left(1-\sqrt{\frac{2}{\pi}}\right)} \approx 0.6358.
\]
\end{theorem}

\paragraph{How large is this discrepancy?}
The lower bounds in \cref{cor:negative-concrete} are quantitatively nontrivial. Since both
distributions are supported on the unit sphere $\sphere$, their Wasserstein distance is always
at most the diameter of the sphere, namely $2$. Thus lower bounds such as $1/3$ for all
power-of-two dimensions $d\ge 256$ and $0.6$ for all power-of-two dimensions $d\ge 32768$
are far from negligible. In fact, the lower bound does not merely stay positive: in the
high-dimensional limit it remains bounded away from zero by at least a constant of about $0.636$. In particular, if the structured
rotation were a good global approximation of the uniform spherical law, one would expect the
corresponding Wasserstein distance to tend to zero as $d\to\infty$. The corollary shows that
this is false. This stands in sharp contrast with the positive result for one-dimensional
marginals, where the approximation error does vanish with the dimension.

\begin{proof}[Proof of \cref{cor:negative-concrete}]
We begin with the first claim. Apply \cref{thm:negative-explicit} with $t=0.11$. For
every power-of-two $d\ge 256$, the quantity
\(
m_d=\sqrt{\frac{2}{\pi}}\sqrt{\frac{d}{d-1}}
\)
is decreasing in $d$, and therefore
\[
1-m_d\ge 1-m_{256} = 1 - \sqrt{\frac{2}{\pi}}\sqrt{\frac{256}{255}} \approx 0.20055 > 0.11,
\]
so the choice $t=0.11$ is admissible for every power-of-two $d\ge 256$.

For fixed $t$, the lower-bound expression
\[
\sqrt{2(1-m_d-t)}\,(1-e^{-(d-1)t^2/2})
\]
is increasing in $d$, because $m_d$ decreases in $d$ and $1-e^{-(d-1)t^2/2}$ increases in
$d$. Therefore, for every power-of-two $d\ge 256$,
\[
\sup_{u\in \sphere}\Wone\!\left(\Tstr(u),X(u)\right)
\ge
\sqrt{2(1-m_{256}-0.11)}\,
\left(1-e^{-255\cdot (0.11)^2/2}\right).
\]
A direct calculation shows that the right-hand side is approximately $0.3346$, hence it is
strictly larger than $1/3$.

For the second claim, apply \cref{thm:negative-explicit} with $t=0.02$. For every
power-of-two $d\ge 32768$,
\[
1-m_d\ge 1-m_{32768} = 1 - \sqrt{\frac{2}{\pi}}\sqrt{\frac{32768}{32767}} \approx 0.20210>0.02,
\]
so the choice $t=0.02$ is admissible for every power-of-two $d\ge 32768$. As above, the
lower-bound expression is increasing in $d$ for fixed $t$, and therefore for every
power-of-two $d\ge 32768$,
\[
\sup_{u\in \sphere}\Wone\!\left(\Tstr(u),X(u)\right)
\ge
\sqrt{2(1-m_{32768}-0.02)}\,
\left(1-e^{-32767\cdot (0.02)^2/2}\right).
\]
A direct calculation shows that the right-hand side is approximately $0.6026$, and in
particular it is strictly larger than $0.6$.
Finally, let $\alpha>0$ and set $t=\frac{\alpha}{\sqrt{d-1}}.$
To apply \cref{thm:negative-explicit}, we must check that
\[
t\in[0,1-m_d],
\qquad\text{that is,}\qquad
\frac{\alpha}{\sqrt{d-1}}\le 1-m_d.
\]
Since
\[
m_d=\sqrt{\frac{2}{\pi}}\sqrt{\frac{d}{d-1}}
\longrightarrow \sqrt{\frac{2}{\pi}}
\qquad\text{as } d\to\infty,
\]
we have
\[
1-m_d-\frac{\alpha}{\sqrt{d-1}}\longrightarrow 1-\sqrt{\frac{2}{\pi}}>0.
\]
Therefore, for every fixed $\alpha>0$, there exists $d_0(\alpha)$ such that for every
power-of-two dimension $d\ge d_0(\alpha)$,
\[
\frac{\alpha}{\sqrt{d-1}}\le 1-m_d.
\]
Hence, for all sufficiently large powers of two $d$, the choice $t=\tfrac{\alpha}{\sqrt{d-1}}$
is admissible in \cref{thm:negative-explicit}. Applying the theorem, we obtain
\[
\sup_{u\in \sphere}\Wone\!\left(\Tstr(u),X(u)\right)
\ge
\sqrt{2\left(1-m_d-\frac{\alpha}{\sqrt{d-1}}\right)}\,(1-e^{-\alpha^2/2}).
\]
Taking the lower limit as $d\to\infty$ along powers of two gives
\[
\liminf_{d\to\infty,\ d\in\{2^m:m\in\mathbb N\}}
\sup_{u\in \sphere}\Wone\!\left(\Tstr(u),X(u)\right)
\ge
\sqrt{2\left(1-\sqrt{\frac{2}{\pi}}\right)}\,(1-e^{-\alpha^2/2}).
\]
Since this holds for every $\alpha>0$, we let $\alpha\to\infty$ and conclude that
\[
\liminf_{d\to\infty,\ d\in\{2^m:m\in\mathbb N\}}
\sup_{u\in \sphere}\Wone\!\left(\Tstr(u),X(u)\right)
\ge
\sqrt{2\left(1-\sqrt{\frac{2}{\pi}}\right)}.
\]
\end{proof}

\subsection[Roadmap of the proof of Theorem~\ref{thm:negative-explicit}]{Roadmap of the proof of \cref{thm:negative-explicit}}
The proof of the explicit Wasserstein lower bound uses a concrete input and  a $1$-Lipschitz test function.

\smallskip
\noindent\textbf{Step 1: choose a specific input.}
We take $u=e_1=(1,0,0,\dots)^\top$. For this input, the \emph{two-stage} structured rotation
\[
\Tstr(e_1)=\frac{1}{d}\Had \DOne \Had \DTwo e_1
\]
produces a distribution supported on a rotated hypercube. By contrast, a single structured
rotation applied to $e_1$ would be supported on only two points.

\smallskip
\noindent\textbf{Step 2: remove the outer Hadamard transform.}
For $u=e_1$, the structured output can be written as
\[
\Tstr(e_1)=\frac{1}{\sqrt d}\Had Y,
\qquad
Y\sim \Unif\!\left(\left\{\pm \frac{1}{\sqrt d}\right\}^d\right).
\]
Thus the structured law is the image of the uniform law on the vertices of the scaled
hypercube under the orthogonal map $d^{-1/2}\Had$. Since orthogonal maps preserve Euclidean
distances, applying $d^{-1/2}\Had$ to both random vectors does not change Wasserstein
distance. This reduces the problem to comparing the uniform spherical law directly with the
uniform law on the scaled hypercube vertices.

\smallskip
\noindent\textbf{Step 3: choose an explicit separating function.}
We take as our test function the Euclidean distance to the vertex set
\[
\mathcal C_d=\left\{\pm \frac{1}{\sqrt d}\right\}^d.
\]
This function is $1$-Lipschitz, and it vanishes identically on the hypercube law. Therefore,
by the dual representation of Wasserstein distance, the problem reduces to lower-bounding its
expectation under the uniform spherical law. For points on the sphere, this distance can be
computed explicitly: it is determined by the largest inner product with a hypercube vertex,
which in turn equals the normalized $\ell_1$ norm.

\smallskip
\noindent\textbf{Step 4: use concentration of the normalized $\ell_1$ norm.}
The explicit formula from the previous step shows that the separating function is large
whenever $\|X\|_1/\sqrt d$
is not too large. We therefore study this normalized $\ell_1$ norm under the uniform
spherical law. Its expectation is bounded above by $m_d$, and since it is a $1$-Lipschitz
function on the sphere, concentration of measure implies that it exceeds $m_d+t$ only with
exponentially small probability. On the complementary high-probability event, the separating
function is bounded below by an explicit positive quantity, which yields the Wasserstein lower
bound.

We state the auxiliary lemmas here and provide complete proofs in \cref{app:negative-details} of the Appendix.

\begin{lemma}\label{lem:had-invariance}
If $U$ is distributed uniformly on $\sphere$, then
$\Had U/\sqrt d \stackrel{d}{=} U.$
Moreover, for every random vector $W$ supported on $\sphere$,
\[
\Wone\!\left(U,\frac{1}{\sqrt d}\Had W\right)=\Wone(U,W).
\]
\end{lemma}

\begin{lemma}\label{lem:l1-lip}
The function
$g(x)\coloneqq \left\|x\right\|_1 / \sqrt {d},
 x\in\R^d,$
is $1$-Lipschitz with respect to the Euclidean norm.
\end{lemma}

\begin{lemma}[Spherical concentration]\label{lem:spherical-concentration}
Let $X$ be uniformly distributed on $\sphere$, and let $h:\sphere\to\R$ be $1$-Lipschitz
with respect to the geodesic distance on $\sphere$. Then for every $t>0$,
\[
\Prob\!\left( h(X)-\E h(X) >t\right)\le e^{-(d-1)t^2/2}.
\]
\end{lemma}
This can be seen in \cite[Section 5.1]{ledoux2001concentration}. Specifically, Ledoux states in Equation (5.7) that the uniform distribution on the sphere $\sphere$ has Log-Sobolev Inequality with the constant $\frac{1}{d-1}$. Using this fact together with Theorem 5.3 from that same section gives the above lemma.

\begin{lemma}\label{lem:coord-mean}
Let $U$ be uniformly distributed on $\sphere$, and let $U_1$ denote its first entry. Then
\[
\E|U_1| = \frac{\Gamma(d/2)}{\sqrt{\pi}\,\Gamma((d+1)/2)}
\le \sqrt{\frac{2}{\pi d}}\sqrt{\frac{d}{d-1}}.
\]
Consequently,
\[
\E\!\left[\frac{\left\|U\right\|_1}{\sqrt d}\right]
=
\sqrt d\,\E|U_1|
\le
\sqrt{\frac{2}{\pi}}\sqrt{\frac{d}{d-1}}
=
m_d.
\]
\end{lemma}

\begin{lemma}\label{lem:distance-to-set-lip}
Let $A\subseteq \R^d$ be nonempty and define
$f_A(x)\coloneqq \inf_{a\in A}\norm{x-a}.$
Then $f_A$ is $1$-Lipschitz on $\R^d$.
\end{lemma}

\subsection{Proof of the negative theorem}

\begin{proof}[Proof of \cref{thm:negative-explicit}]
We choose the specific input vector
$
u=e_1=(1,0,\dots,0)^\top.
$
 Since the theorem takes a supremum over all $u\in\sphere$, it is enough to prove the lower bound for this particular choice.

Let
\[
W \coloneqq \Tstr(e_1)=\frac{1}{d}\Had \DOne \Had \DTwo e_1.
\]
Because the first column of the Walsh--Hadamard matrix consists entirely of ones,
\[
\Had \DTwo e_1 = D^{(2)}_{11}\,\bm 1,
\]
where $\bm 1=(1,\dots,1)^\top$. Hence
\[
W
=
\frac{D^{(2)}_{11}}{d}\Had \DOne \bm 1.
\]
The vector
$(D^{(1)}_{11},\dots,D^{(1)}_{dd})$
is uniformly distributed on $\{\pm 1\}^d$, because its coordinates are independent
Rademacher random variables. Multiplying all coordinates by the additional sign
$D^{(2)}_{11}\in\{\pm 1\}$ does not change this distribution. Therefore, writing

\[
Y\coloneqq \frac{D^{(2)}_{11}}{\sqrt d}\DOne\bm 1,
\]

we have
\[
Y\sim \Unif\!\left(\left\{\pm \frac{1}{\sqrt d}\right\}^d\right)
\qquad\text{and}\qquad
W = \frac{1}{\sqrt d}\Had Y.
\]

Now let
\[
X=X(e_1)=Re_1,
\]
so $X$ is uniformly distributed on $\sphere$.

By \cref{lem:had-invariance}, the random vector $d^{-1/2}\Had X$ has the same distribution as
$X$. Also, because $d^{-1/2}\Had$ is orthogonal, it preserves Euclidean distances and hence
preserves Wasserstein distance when applied to both random vectors. Therefore
\[
\Wone(X,W)
=
\Wone\!\left(\frac{1}{\sqrt d}\Had X,\frac{1}{\sqrt d}\Had W\right)
=
\Wone\!\left(X,\frac{1}{\sqrt d}\Had W\right).
\]
Since
\[
W=\frac{1}{\sqrt d}\Had Y
\qquad\text{and}\qquad
\left(\frac{1}{\sqrt d}\Had\right)^2=I_d,
\]
we obtain
\[
\Wone(X,W) = \Wone\left(X, \frac{1}{\sqrt d}HW\right) = \Wone(X, Y).
\]
Thus it is enough to lower-bound $\Wone(X,Y)$.

Define
\[
f(x)\coloneqq \mathrm{dist}(x,\mathcal C_d)
=
\inf_{y\in \mathcal C_d}\|x-y\|_2,
\qquad x\in\R^d.
\]
By \cref{lem:distance-to-set-lip}, the function $f$ is $1$-Lipschitz. Since $Y\in \mathcal C_d$
almost surely, we have $f(Y)=0$.

By the Kantorovich--Rubinstein dual representation of Wasserstein distance (see, for example,
\cite{Villani2009}, Equation 6.3), we have
\[
\Wone(X,Y)=\sup_{h\in\Lip(1)}\abs{\E[h(X)]-\E[h(Y)]}.
\]
Since $f$ is $1$-Lipschitz, it is an admissible test function in this supremum. Therefore
\[
\Wone(X,Y)\ge \abs{\E[f(X)]-\E[f(Y)]}.
\]
Because $f(Y)=0$ almost surely, this simplifies to
\[
\Wone(X,Y)\ge \E[f(X)].
\]

We now compute $f(x)$ explicitly for $x\in \sphere$. Fix $x\in \sphere$. For any
$y\in\mathcal C_d$,
\[
\norm{x-y}^2 = \norm{x}^2+\norm{y}^2 - 2\ip{x}{y}
= 2 - 2\ip{x}{y},
\]
because both $x$ and $y$ are on the unit sphere and therefore have Euclidean norm $1$.

Thus minimizing $\norm{x-y}$ over $y\in \mathcal C_d$ is equivalent to maximizing the inner
product $\ip{x}{y}$ over $y\in\mathcal C_d$. Since each coordinate of a vector in
$\mathcal C_d$ equals either $1/\sqrt d$ or $-1/\sqrt d$, we have
\[
\ip{x}{y}
=
\sum_{i=1}^d ( x_i y_i)
\le
\frac{1}{\sqrt d}\sum_{i=1}^d |x_i|
=
\frac{\left\|x\right\|_1}{\sqrt d}.
\]
Equality is attained by choosing each sign to match the sign of the corresponding coordinate
of $x$, namely
\[
y_i=\frac{\operatorname{sign}(x_i)}{\sqrt d},
\]
with an arbitrary choice when $x_i=0$. Therefore
\[
\max_{y\in \mathcal C_d}\ip{x}{y}
=
\frac{\left\|x\right\|_1}{\sqrt d},
\]
and hence
\[
f(x)=\sqrt{2-2\frac{\left\|x\right\|_1}{\sqrt d}}.
\]
Applying this identity to $X$ gives
\[
\E[f(X)]
=
\sqrt2\,\E\!\left[\sqrt{1-\frac{\left\|X\right\|_1}{\sqrt d}}\right].
\]

Now define
\[
g(x)\coloneqq \frac{\left\|x\right\|_1}{\sqrt d},
\qquad x\in \sphere.
\]
By \cref{lem:l1-lip}, the map $g$ is $1$-Lipschitz with respect to the Euclidean metric on
$\R^d$. To apply \cref{lem:spherical-concentration}, we need to check that its restriction to
$\sphere$ is also $1$-Lipschitz with respect to the geodesic distance on $\sphere$.

For any $x,y\in \sphere$, the geodesic distance between $x$ and $y$ is at least their
Euclidean distance. Therefore
\[
|g(x)-g(y)|\le \|x-y\|_2 \le d_{\mathrm{geo}}(x,y),
\]
so the restriction of $g$ to $\sphere$ is indeed $1$-Lipschitz with respect to the geodesic
metric. We may therefore apply \cref{lem:spherical-concentration} to $g(X)$ and obtain
\[
\Prob\!\left(g(X)-\E g(X)>t\right)\le e^{-(d-1)t^2/2}
\qquad\text{for every }t>0.
\]
By \cref{lem:coord-mean},
\[
\E g(X)\le m_d.
\]
Hence, for every $t\in[0,1-m_d]$,
\[
\Prob\!\left(g(X)\le m_d+t\right)\ge 1-e^{-(d-1)t^2/2}.
\]
On the event $\{g(X)\le m_d+t\}$ we have
\[
f(X)=\sqrt{2(1-g(X))}\ge \sqrt{2(1-m_d-t)}.
\]
Therefore
\begin{align*}
\E[f(X)]
&\ge
\sqrt{2(1-m_d-t)}\,
\Prob\!\left(g(X)\le m_d+t\right) \\
&\ge
\sqrt{2(1-m_d-t)}\,(1-e^{-(d-1)t^2/2}).
\end{align*}
Combining the inequalities above, we conclude that
\[
\Wone\!\left(\Tstr(e_1),X(e_1)\right)
=
\Wone(X,Y)
\ge
\sqrt{2(1-m_d-t)}\,(1-e^{-(d-1)t^2/2}).
\]
Since the left-hand side is bounded above by
\[
\sup_{u\in\sphere}\Wone\!\left(\Tstr(u),X(u)\right),
\]
the theorem follows.
\end{proof}

\subsection{A matching upper bound for the same input}

The lower bound above was proved by specializing to the input $u=e_1$ and then choosing a
particular $1$-Lipschitz test function, namely the distance to the scaled hypercube vertex set.
At that stage, one may wonder whether this test function is close to optimal, or whether a
different $1$-Lipschitz function could yield a substantially larger lower bound for the same
input. The next corollary shows that this is not the case: for $u=e_1$, one can prove a clean
upper bound on the full Wasserstein distance itself, and this upper bound matches the lower-bound
constant asymptotically. Consequently, for this specific input, the lower bound proved above is
asymptotically sharp. In other words, for $u=e_1$, it essentially captures the true Wasserstein
distance in high dimensions. For other inputs, however, the exact asymptotic Wasserstein distance
remains unknown, and understanding it would be an interesting direction for further study.

\begin{theorem}[Upper bound for the same input]\label{thm:negative-upper-same-input}
Let
\[
u=e_1=(1,0,\dots,0)^\top,
\qquad
X\coloneqq X(e_1)=Re_1,
\qquad
W\coloneqq \Tstr(e_1).
\]
Then
\[
\Wone(X,W)
=
\E\left\|X-\frac{1}{\sqrt d}\operatorname{sign}(X)\right\|_2
\le
\sqrt{2-\frac{2}{\sqrt d}\,\E\|X\|_1}
\le
\sqrt{2\left(1-\sqrt{\frac{2}{\pi}}\right)}.
\]
In particular,
\[
\limsup_{d\to\infty,\ d\in\{2^m:m\in\mathbb N\}}
\Wone\!\left(X(e_1),\Tstr(e_1)\right)
\le
\sqrt{2\left(1-\sqrt{\frac{2}{\pi}}\right)}.
\]
\end{theorem}

\begin{proof}[Proof of \cref{thm:negative-upper-same-input}]
Let
$Y\sim \Unif(\mathcal C_d).$
From the proof of \cref{thm:negative-explicit}, we already know that for $u=e_1$,
$\Wone(X,W)=\Wone(X,Y).$
Thus it is enough to bound $\Wone(X,Y)$.

We first show that
\[
\Wone(X,Y)=\E\left\|X-\frac{1}{\sqrt d}\operatorname{sign}(X)\right\|_2.
\]
Recall
$f(x)\coloneqq \mathrm{dist}(x,\mathcal C_d)
=
\inf_{y\in \mathcal C_d}\|x-y\|_2, 
 x\in\R^d.$
In the proof of \cref{thm:negative-explicit}, we showed that for every $x\in \sphere$,
the nearest point in $\mathcal C_d$ is
$\frac{1}{\sqrt d}\operatorname{sign}(x),$
with an arbitrary choice of sign on zero coordinates, and therefore
\[
f(x)=\left\|x-\frac{1}{\sqrt d}\operatorname{sign}(x)\right\|_2.
\]
Since $f$ is $1$-Lipschitz and $Y\in\mathcal C_d$ almost surely, the same duality argument
used in the lower bound gives
\[
\Wone(X,Y)\ge \E[f(X)]
=
\E\left\|X-\frac{1}{\sqrt d}\operatorname{sign}(X)\right\|_2.
\]

For the reverse inequality, we use a coupling. Let
\[
Y^\ast\coloneqq \frac{1}{\sqrt d}\operatorname{sign}(X).
\]
We claim that $Y^\ast$ is uniformly distributed on $\mathcal C_d$. Indeed, write
\[
X=\frac{G}{\|G\|_2},
\qquad
G=(G_1,\dots,G_d)\sim \Normal(0,I_d).
\]
Then
$\operatorname{sign}(X)=\operatorname{sign}(G),$
and since the coordinates of $G$ are independent centered Gaussians, the sign vector
$\operatorname{sign}(G)\in\{\pm 1\}^d$
is uniformly distributed on $\{\pm 1\}^d$. Hence $Y^\ast\sim \Unif(\mathcal C_d)$, exactly
like $Y$. The pair $(X,Y^\ast)$ is therefore a valid coupling of the two distributions. Since Wasserstein
distance is defined as the infimum of $\E\|U-V\|_2$ over all such couplings $(U,V)$, evaluating
this quantity at the particular coupling $(X,Y^\ast)$ yields the upper bound
\[
\Wone(X,Y)\le \E\|X-Y^\ast\|_2=
\E\left\|X-\frac{1}{\sqrt d}\operatorname{sign}(X)\right\|_2.
\]

Combining the lower and upper bounds yields
\[
\Wone(X,Y)=\E\left\|X-\frac{1}{\sqrt d}\operatorname{sign}(X)\right\|_2.
\]

We now derive the explicit upper bound. By Jensen's inequality,
\[
\Wone(X,Y)
=
\E\left\|X-\frac{1}{\sqrt d}\operatorname{sign}(X)\right\|_2
\le
\sqrt{
\E\left\|X-\frac{1}{\sqrt d}\operatorname{sign}(X)\right\|_2^2 }.
\]
Expanding the square gives
\begin{align*}
\E\left\|X-\frac{1}{\sqrt d}\operatorname{sign}(X)\right\|_2^2
&=
\E\|X\|_2^2
-\frac{2}{\sqrt d}\E\!\left[X^\top \operatorname{sign}(X)\right]
+\E\left\|\frac{1}{\sqrt d}\operatorname{sign}(X)\right\|_2^2 \\
&=
1-\frac{2}{\sqrt d}\E\|X\|_1+1 
=
2-\frac{2}{\sqrt d}\E\|X\|_1.
\end{align*}
Hence
\[
\Wone(X,Y)\le \sqrt{2-\frac{2}{\sqrt d}\E\|X\|_1}.
\]

It remains to bound $\E\|X\|_1$ from below. By exchangeability,
$\E\|X\|_1=d\,\E|X_1|.$
For a uniform random vector on $\sphere$, one coordinate satisfies
\[
\E|X_1|=\frac{\Gamma(d/2)}{\sqrt\pi\,\Gamma((d+1)/2)}.
\]
Set
$x=\frac d2,
s=\frac12.$
By Gautschi's inequality, and using $\Gamma(x+1)=x\Gamma(x)$, we obtain:
\[
x^{1-s}<\frac{\Gamma(x+1)}{\Gamma(x+s)} =\frac{x\Gamma(x)}{\Gamma(x+s)} \iff \left(\frac{d}{2}\right)^{-(1/2)}<\frac{\Gamma(d/2)}{\Gamma((d+1)/2)} .
\]
Consequently,
\[
\E|X_1|>\sqrt{\frac{2}{\pi d}},
\qquad\text{and hence}\qquad
\E\|X\|_1>\sqrt{\frac{2d}{\pi}}.
\]
Substituting this into the previous bound gives
\[
\Wone(X,Y)\le
\sqrt{2-\frac{2}{\sqrt d}\sqrt{\frac{2d}{\pi}}}
=
\sqrt{2\left(1-\sqrt{\frac{2}{\pi}}\right)}.
\]
Finally, since $\Wone(X,W)=\Wone(X,Y)$, the same bound holds for $\Wone(X,W)$, which proves
the Theorem.

\end{proof}

\section{Scope and implications of the distributional results}
\label{sec:implications}

\Cref{thm:positive} bounds individual normalized coordinate laws,
while \cref{sec:negative} shows a separation between the full,
unscaled vector laws. The implications depend on the statistic
being compared.

\subsection{Expected averages of coordinate statistics}

The bounds in \cref{sec:positive} concern individual coordinate laws,
so they do not generally control quantities that depend on the joint
law. Nevertheless, a statistic may use several, or even all,
coordinates while its expectation depends only on their individual
marginals. Expected averages of coordinatewise threshold indicators
and scalar $1$-Lipschitz losses are two examples. Fix $u\in\sphere$, choose thresholds
$a_1,\ldots,a_d$ and $1$-Lipschitz functions
$\phi_1,\ldots,\phi_d$, and define, for $z\in\R^d$,
\[
A_a(z)
=\frac1d\sum_{k=1}^d
  \bbone\{\sqrt d\,z_k>a_k\},
\qquad
L_\phi(z)
=\frac1d\sum_{k=1}^d
  \phi_k(\sqrt d\,z_k).
\]
Here $A_a$ is the fraction of normalized coordinates exceeding
their thresholds, and $L_\phi$ is an average of scalar losses.
Applying \cref{thm:positive} to each term and averaging gives
\begin{equation}\label{eq:implications-coordinate-averages}
\begin{aligned}
\abs{\E A_a(\Tstr(u))-\E A_a(X(u))}
&\le \Delta_{\mathrm K}(d),\\
\abs{\E L_\phi(\Tstr(u))-\E L_\phi(X(u))}
&\le \Delta_{\mathrm W}(d).
\end{aligned}
\end{equation}
The argument uses linearity of expectation and requires no
independence between coordinates. This explains
why a statistic involving every coordinate can be controlled in
expectation despite the full-vector separation in
\cref{sec:negative}.

\subsection{Limits and questions for applications}

The marginal bounds do not establish approximation of the joint
law of two coordinates or of a larger subsample. Nor does
\eqref{eq:implications-coordinate-averages} establish
concentration or compare the distributions of the averages.
For example, the bound on the expected value of $A_a$ does not
by itself give a vanishing bound on
\[
\abs{\Prob(A_a(\Tstr(u))>q)-\Prob(A_a(X(u))>q)}
\]
for a prescribed $q$. Conversely, the full-vector separation
does not establish a discrepancy in any specified algorithmic
performance measure.

This paper establishes distributional consequences and leaves
their transfer to specific algorithms open. Such a transfer
requires identifying the event or statistic that expresses an
algorithm's guarantee and determining whether the present bounds
control it.

\section{Numerical illustrations}\label{sec:numerics}

We complement the theory with three numerical figures. These are not intended
as a substitute for proofs. Their purpose is to visualize the two phenomena
identified above: improvement of one-dimensional behavior and persistence of
global discrepancy.

\subsection{Coordinate distances and upper bounds}

Figure~\ref{fig:ks-pointwise-ci} compares the exact first-coordinate
Kolmogorov distance for $u=e_1$ with the uniform upper bound from
\cref{thm:positive}. The distance is computed directly from the binomial
coordinate law. Panel (a) shows the distance and bound
on their original scale. Panel (b) divides each curve by its own value
at dimension 16 to compare their relative decay.

\begin{figure}[htbp]
\centering
\includegraphics[width=0.80\textwidth]{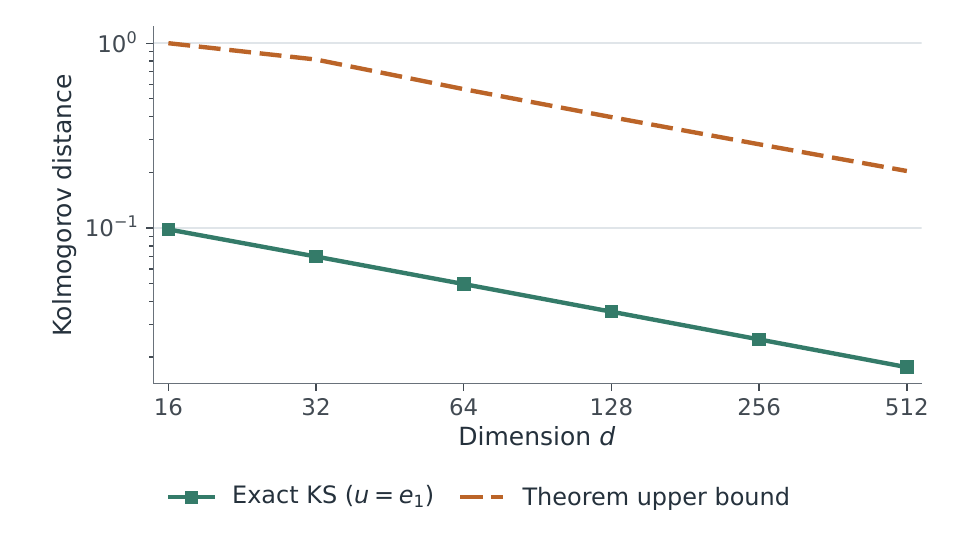}
\par\small (a) Distances and upper bounds.

\vspace{0.6em}

\includegraphics[width=0.80\textwidth]{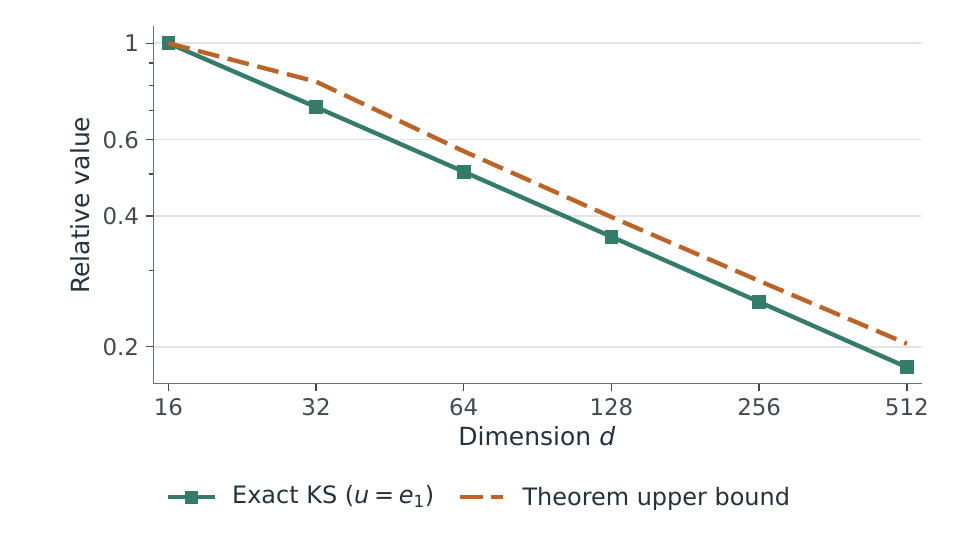}
\par\small (b) Each curve divided by its value at dimension 16.

\caption{Coordinate comparisons against the exact uniform-spherical
coordinate CDF. In (a), green squares give the exact
Kolmogorov distance for $u=e_1$, computed by checking
both sides of every atom $x_j=(2j-d)/d$ with mass
$2^{-d}\binom{d}{j}$. The orange dashed curve is the uniform theorem bound
$B_d=\min\{1,\,0.56\sqrt{3\log d/d}+2/\sqrt d+4/(d-4)\}$,
using its stated constants. In (b), both curves are shown
as $q(d)/q(16)$: the denominator is fixed at dimension 16,
while the horizontal axis varies over $d$.}
\label{fig:ks-pointwise-ci}
\end{figure}
\FloatBarrier

\subsection{Marginal upper and lower bounds}

Figure~\ref{fig:marginal-bounds} plots the upper and lower bounds
from \cref{thm:positive} as functions of dimension.

\begin{figure}[htbp]
\centering
\begin{minipage}[t]{0.48\textwidth}
\centering
\includegraphics[width=\linewidth]{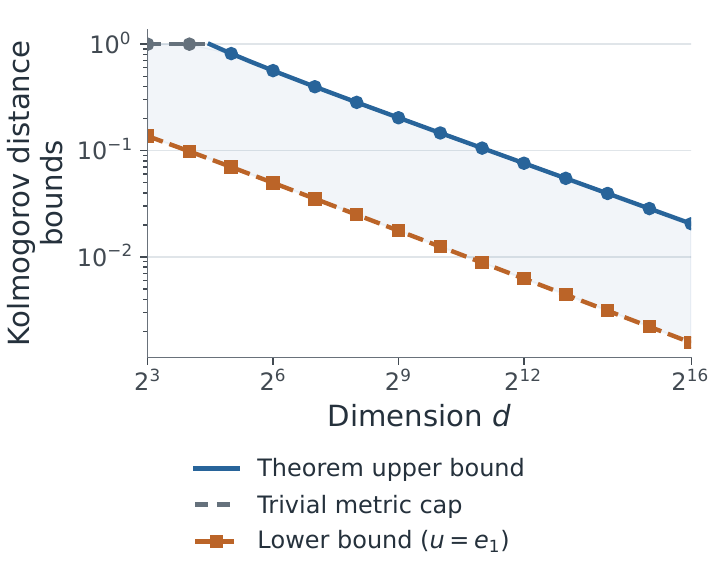}
\par\small (a) Kolmogorov marginal bounds.
\end{minipage}\hfill
\begin{minipage}[t]{0.48\textwidth}
\centering
\includegraphics[width=\linewidth]{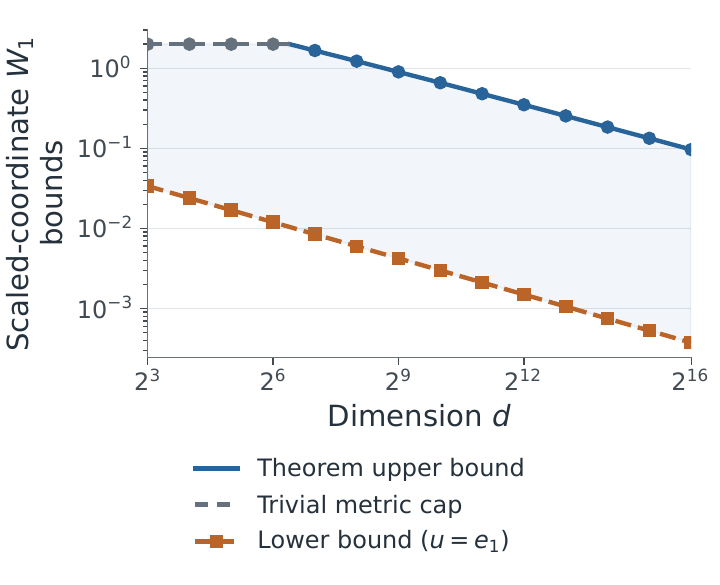}
\par\small (b) Scaled-coordinate Wasserstein-1 bounds.
\end{minipage}
\caption{Analytic bounds on the worst-case one-coordinate distances.
For Wasserstein-1, both coordinates are multiplied by $\sqrt d$;
Kolmogorov distance is unchanged by this common scaling.
The upper curves use $s=\sqrt{3\log d/d}$, with
$C_{\mathrm{BE}}=0.56$ for Kolmogorov distance. 
The lower curves use $u=e_1$: $2^{-d-1}\binom{d}{d/2}$ for
Kolmogorov distance and $\sqrt{5/(64\pi e d)}$ for Wasserstein-1.
}
\label{fig:marginal-bounds}
\end{figure}
\FloatBarrier

\subsection{Full-vector separation}

Figure~\ref{fig:lower} compares the full-vector bounds for $u=e_1$
with a numerical evaluation of the exact same-input expectation from
\cref{thm:negative-upper-same-input}.

\begin{figure}[htbp]
\centering
\begin{minipage}[t]{0.48\textwidth}
\centering
\includegraphics[width=\linewidth]{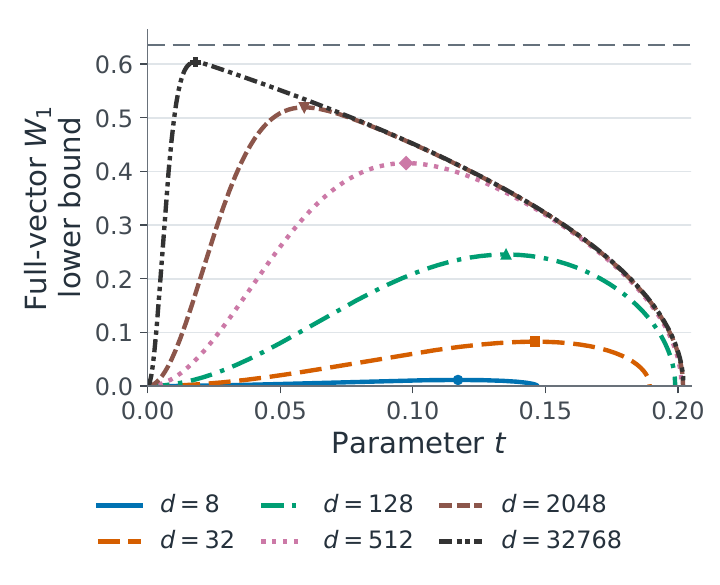}
\par\small (a) Parameterized lower bound.
\end{minipage}\hfill
\begin{minipage}[t]{0.48\textwidth}
\centering
\includegraphics[width=\linewidth]{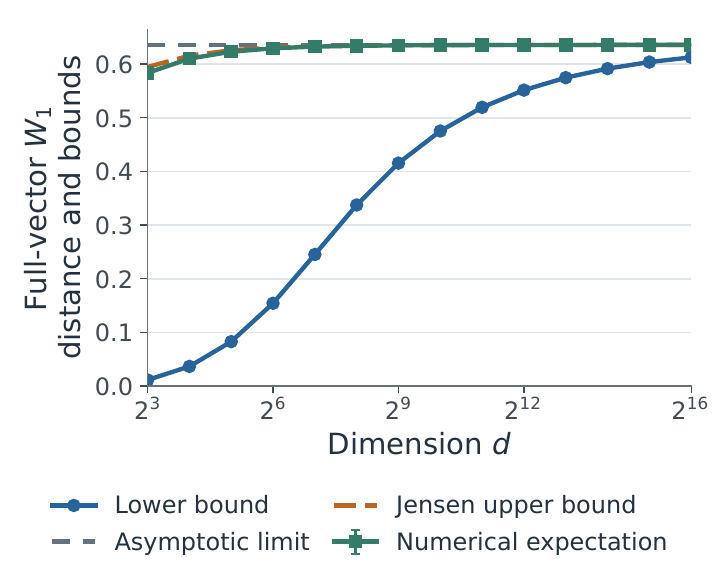}
\par\small (b) Same-input distance and bounds for $u=e_1$.
\end{minipage}
\caption{Full-vector Wasserstein-1 distance and bounds with Euclidean
transport cost for $u=e_1$. In (a),
$L_d(t)=\sqrt{2(1-m_d-t)}(1-e^{-(d-1)t^2/2})$,
where $m_d=\sqrt{2/\pi}\sqrt{d/(d-1)}$ and $0\le t\le1-m_d$;
markers identify maximizing thresholds. In (b), blue circles give the
optimized lower bound, the orange dashed curve is the finite-dimensional
Jensen upper bound, and green squares numerically evaluate
$\E\sqrt{2-2\|X\|_1/\sqrt d}$ for $X$ uniform on $\sphere$.
This expectation equals the exact same-input Wasserstein-1 distance.
Its numerical integration uses seeded Monte Carlo with an exact-moment
control variate; approximate pointwise $95\%$ Monte Carlo confidence
intervals are smaller than the markers.
The full-vector upper bounds concern this specific input.}
\label{fig:lower}
\end{figure}
\FloatBarrier

\section{Acknowledgements}
The LLM GPT-5.5 was used to help polish text in this paper.

\appendix

\section{Auxiliary details for the negative theorem}\label{app:negative-details}

\subsection{Hadamard invariance}

\begin{proof}[Proof of \cref{lem:had-invariance}]
Let $Q\coloneqq \Had/ \sqrt d.$
Then $Q$ is an orthogonal matrix.
If $U$ is uniformly distributed on $\sphere$, then for every orthogonal matrix $O$, the random
vector $OU$ is also uniformly distributed on $\sphere$. Applying this with $O=Q$ gives
$QU \stackrel{d}{=} U$
which proves the first claim.

For the second claim, let $W$ be any random vector supported on $\sphere$. Since $Q$ is
orthogonal, it preserves Euclidean distances:
\[
\norm{Qx-Qy}=\norm{x-y}
\qquad\text{for all }x,y\in\R^d.
\]
Hence, if $h$ is $1$-Lipschitz on $\R^d$, then so is $h\circ Q$. By the dual form of
Wasserstein distance,
\begin{align*}
\Wone(QU,QW)
&=
\sup_{h\in\Lip(1)}\abs{\E[h(QU)]-\E[h(QW)]} \\
&=
\sup_{h\in\Lip(1)}\abs{\E[(h\circ Q)(U)]-\E[(h\circ Q)(W)]} \le
\Wone(U,W).
\end{align*}
Applying the same argument with $Q^{-1}=Q^\top$ in place of $Q$ gives the reverse inequality,
so in fact
$\Wone(QU,QW)=\Wone(U,W).$
Since $QU\stackrel{d}{=}U$, we conclude that
$\Wone\!\left(U,QW\right)=\Wone(U,W).$
This proves the lemma.
\end{proof}

\subsection{\texorpdfstring{$\ell_1$}{l1} Lipschitz property}

\begin{proof}[Proof of \cref{lem:l1-lip}]
For any $x,y\in\R^d$,
$\abs{\left\|x\right\|_1-\left\|y\right\|_1}\le \left\|x-y\right\|_1$
by the reverse triangle inequality for the norm $\left\|\cdot\right\|_1$. Hence
\[
|g(x)-g(y)|
=
\frac{1}{\sqrt d}\abs{\left\|x\right\|_1-\left\|y\right\|_1}
\le
\frac{1}{\sqrt d}\left\|x-y\right\|_1.
\]
By Cauchy--Schwarz,
$\left\|x-y\right\|_1\le \sqrt d\,\norm{x-y}.$
Substituting this estimate above yields
\[
|g(x)-g(y)|\le \norm{x-y}.
\]
Thus $g$ is $1$-Lipschitz with respect to the Euclidean norm.
\end{proof}

\subsection{Mean of the normalized \texorpdfstring{$\ell_1$}{l1} norm}
\begin{proof}[Proof of \cref{lem:coord-mean}]
Let $U$ be uniformly distributed on $\sphere$. The exact formula
\[
\E|U_1| = \frac{\Gamma(d/2)}{\sqrt{\pi}\,\Gamma((d+1)/2)}
\]
is standard for one coordinate of the uniform spherical distribution.

To bound this quantity, we use Gautschi's
inequality~\cite[Equation~5.6.4]{lozier2003nist}, which states that,
for $x>0$ and $0<s<1$,
\[
x^{1-s}
<
\frac{\Gamma(x+1)}{\Gamma(x+s)}
<
(x+1)^{1-s}.
\]
Taking $x=(d-1)/2$ and $s=1/2$ and inverting the lower bound gives
\[
\frac{\Gamma(d/2)}{\Gamma((d+1)/2)}
<
\sqrt{\frac{2}{d-1}}.
\]

therefore
\[
\E|U_1|
=
\frac{\Gamma(d/2)}{\sqrt\pi\,\Gamma((d+1)/2)}
\le
\frac{1}{\sqrt\pi}\sqrt{\frac{2}{d-1}}
=
\sqrt{\frac{2}{\pi d}}\sqrt{\frac{d}{d-1}}.
\]
Finally, by exchangeability of the coordinates of $U$,
\[
\E\!\left[\frac{\left\|U\right\|_1}{\sqrt d}\right]
=
\frac{1}{\sqrt d}\sum_{i=1}^d \E|U_i|
=
\sqrt d\,\E|U_1|\le
\sqrt{\frac{2}{\pi}}\sqrt{\frac{d}{d-1}}
=
m_d.
\]
\end{proof}

\subsection{Distance to a fixed set is Lipschitz}

\begin{proof}[Proof of \cref{lem:distance-to-set-lip}]
Fix $x,y\in\R^d$. For every $a\in A$, the triangle inequality gives
\[
\norm{x-a}\le \norm{x-y}+\norm{y-a}.
\]
So for all $a\in A$, since $f_A(x)$ is the infimum of $\norm{x-a}$
\[
f_A(x) \leq \norm{x-y} + \norm{y-a},
\]
rearranging terms yields
\[
f_A(x) - \norm{x-y} \leq \norm{y - a} .
\]
Note this means $f_A(x) - \norm{x-y}$ is a lower bound for $\{\norm{y-a} : a\in A\}$. By the definition of the infimum as the greatest lower bound, we have 
\[
f_A(x) - \norm{x-y} \leq f_A(y) .
\]
Hence
$f_A(x)-f_A(y)\le \norm{x-y}.$
Exchanging the roles of $x$ and $y$ gives
$f_A(y)-f_A(x)\le \norm{x-y}.$
Therefore
$|f_A(x)-f_A(y)|\le \norm{x-y}.$
So $f_A$ is $1$-Lipschitz.
\end{proof}

\section{Dependence on the transform and random diagonals}
\label{app:transform-flexibility}

Write $F=d^{-1/2}\Had$, so that
$\Tstr(u)=F\DOne F\DTwo u$. The distributional arguments use
\[
F^\top F=I_d
\qquad\text{and}\qquad
|F_{ij}|=d^{-1/2}\quad\text{for every }i,j.
\]
The Walsh recursion is not needed for these identities or for the
probabilistic arguments below; it provides the fast multiplication
used in the motivating construction.

\paragraph{Other real Hadamard matrices}
With the same independent Rademacher diagonals, all our
distributional bounds hold with the same constants for any real
Hadamard matrix normalized by $d^{-1/2}$. Exact flatness gives
the coordinate-spreading estimate and variance-one conditional
sums used for the upper bounds. For $u=e_1$, it also gives the
binomial and lattice marginal law and the rotated-hypercube law
used for the lower bounds and the full-vector analysis.
Orthogonality supplies norm preservation and invariance of the uniform spherical law.
If the matrix is not symmetric, the full-vector proof removes
the outer transform using its transpose.

\paragraph{Other real orthogonal matrices}
Let $Q$ be orthogonal and set $T_Q(u)=Q\DOne Q\DTwo u$.
Conditional on $\DTwo$, the normalized $k$th coordinate has variance
\[
d\sum_{i=1}^d Q_{ki}^2[(Q\DTwo u)_i]^2.
\]
Exact flatness makes this variance one. A weaker bound
$|Q_{ij}|\le C/\sqrt d$ still gives coordinate spreading, but
does not ensure variance one or the binomial, lattice, and
hypercube structures used for the lower bounds. An orthonormal
DCT has bounded but nonflat entries, so the present proofs do
not apply to it unchanged.

\paragraph{Other diagonal distributions and complex transforms}
If $\DTwo$ remains Rademacher, the one-coordinate upper-bound
calculations also allow independent centered, unit-variance real
entries in $\DOne$ with uniformly bounded third absolute moments.
Unless those entries are signs, however, the transform is no longer
orthogonal. The lower bounds use the atom, lattice, and rotated
hypercube generated by Rademacher signs. Changing $\DTwo$ also
requires new spreading and variance estimates.

A normalized Fourier transform with fourth-root-of-unity or
Steinhaus phase diagonals is unitary but complex-valued. Comparison
would require a Haar-unitary reference law and a choice of metric
for complex coordinates; the real Kolmogorov distance used here
does not apply directly. Its support for $u=e_1$ also has different
geometry. We do not establish bounds for these constructions.

\FloatBarrier
\section{Threshold dependence of the marginal upper bounds}
\label{app:marginal-upper-thresholds}

The proofs of the one-coordinate upper bounds give, for every
power-of-two $d\ge8$ and every $s>0$,
\begin{align*}
\Delta_{\mathrm K}(d)
&\le U_{\mathrm K}(d,s)
 :=\min\!\left\{1,\,
 0.56s+2d e^{-ds^2/2}+\frac{4}{d-4}\right\},\\
\Delta_{\mathrm W}(d)
&\le U_{\mathrm W}(d,s)
 :=\min\!\left\{2,\,
 3s+6d e^{-ds^2/2}+\sqrt{\frac{2}{d}}\right\}.
\end{align*}
Here the Wasserstein-1 distance is between coordinates multiplied by
$\sqrt d$. The cap at $1$ is the universal Kolmogorov bound. The cap at $2$
follows from the second moment of each normalized coordinate being $1$:
the triangle inequality and Cauchy--Schwarz give
$\Wone(A,B)\le\E|A|+\E|B|\le2$.
The theorem uses the explicit threshold $s=\sqrt{3\log d/d}$.
Figure~\ref{fig:marginal-upper-thresholds} shows the dependence on $s$
and the thresholds minimizing the parameterized bounds.

\begin{figure}[htbp]
\centering
\begin{minipage}[t]{0.48\textwidth}
\centering
\includegraphics[width=\linewidth]{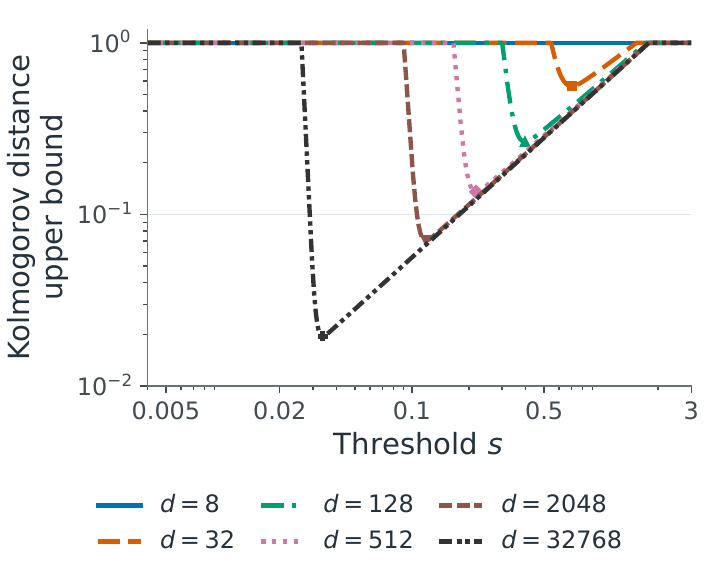}
\par\small(a) Kolmogorov upper bound.
\end{minipage}\hfill
\begin{minipage}[t]{0.48\textwidth}
\centering
\includegraphics[width=\linewidth]{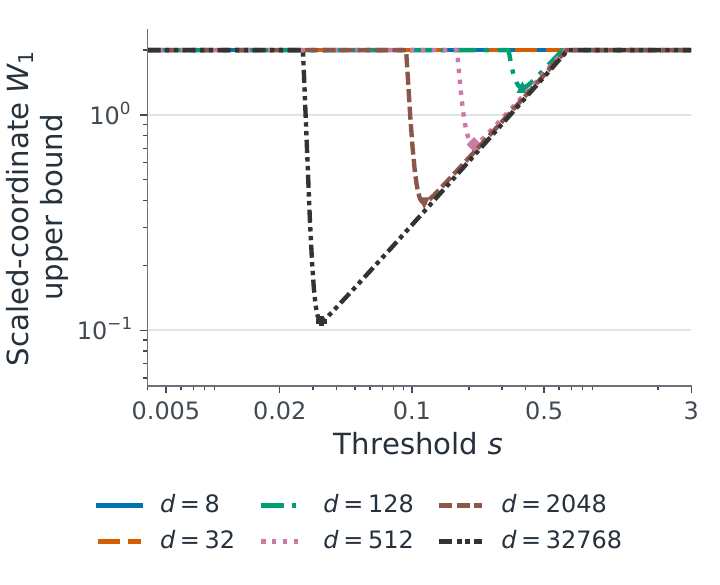}
\par\small(b) Scaled-coordinate Wasserstein-1 upper bound.
\end{minipage}
\caption{Marginal upper bounds as functions of the threshold $s$,
for the indicated dimensions. The curves show $U_{\mathrm K}(d,s)$ and
$U_{\mathrm W}(d,s)$ defined above. Markers identify minimizing thresholds
when the optimized upper bound is strictly below the trivial metric cap.
Flat portions at $1$ and $2$, respectively, give only the trivial bounds;
coincident capped portions can overlap. Both axes are logarithmic.}
\label{fig:marginal-upper-thresholds}
\end{figure}
\FloatBarrier

\bibliographystyle{plain}
\bibliography{references}

\end{document}